\documentclass{article}
\usepackage{iclr2027_conference,times}

\usepackage{amsmath,amsfonts,bm}

\def\eqref#1{equation~\ref{#1}}

\def\1{\bm{1}}

\DeclareMathAlphabet{\mathsfit}{\encodingdefault}{\sfdefault}{m}{sl}
\SetMathAlphabet{\mathsfit}{bold}{\encodingdefault}{\sfdefault}{bx}{n}

\usepackage{amssymb}
\usepackage{amsthm}
\usepackage{thmtools}

\newtheorem{lemma}{Lemma}
\usepackage{url}
\usepackage{booktabs}
\usepackage{algorithm}
\usepackage{algorithmic}
\usepackage{tikz}
\usetikzlibrary{positioning,arrows.meta}
\usepackage{pgfplots}
\pgfplotsset{compat=1.18}
\newcommand{\enzh}[2]{#1} 
\usepackage[caption=false]{subfig}
\usepackage{hyperref}

\title{SpecFit: Not All Draft Tokens Cost the Same\\
in Offloaded MoE Inference}
\renewcommand\footnotemark{}
\author{
    \!Hongliang Cao$^{1,*}$, Xiuyu Li$^{2,*}$, Siqi Wang$^1$, Ziming Wang$^3$, Jingmin Zhu$^2$, Hailong Yang$^{1,\dagger}$
    \thanks{$^*$Equal contribution.}
    \thanks{$^\dagger$Corresponding author.} \\
    $^1$Beihang University \quad $^2$Monash University \\
    $^3$The Hong Kong University of Science and Technology \\
    \footnotesize{\texttt{\{caohongliang,hailong.yang\}@buaa.edu.cn}}
}

\iclrpreprint

\begin{document}
\maketitle

\begin{abstract}
Mixture-of-experts (MoE) models activate only a few experts per token, but when
their expert weights are offloaded to host memory, host-to-device (H2D) expert
transfers enter the critical path of decoding. Speculative decoding verifies
several draft tokens in one forward pass, yet in an MoE model each verified
token activates its own experts, so a larger draft tree transfers more experts.
Because of this H2D cost, speculative decoding commits several tokens per
verification cycle yet is sometimes even slower than autoregressive
decoding. Based on this observation, we propose SpecFit, which chooses the
verified draft subtree to maximize predicted throughput rather than acceptance
length. SpecFit estimates the expected number of committed tokens from draft
path probabilities and predicts cycle time from the node count and the
per-layer union of the experts the subtree activates; this union is a monotone
submodular coverage function, so each node pays only for the experts it adds.
An exact search over ancestor-closed subtrees returns the subtree with the
highest predicted throughput; run after each draft level, the same search
also decides when to stop drafting, and the unmodified target verifies the
result. On
CNN/DailyMail summarization and a GSM8K subset with Qwen3-30B-A3B, SpecFit
decodes 21.6\% and 7.0\% faster than EAGLE-3 with the same expert cache and
1.97$\times$ and 1.72$\times$ as fast as vanilla EAGLE-3. Code is available at \url{https://github.com/Chlience/SpecFit}.
\end{abstract}

\section{Introduction}
\label{sec:intro}

Mixture-of-experts (MoE) layers route each token to a few expert feed-forward
networks, decoupling the parameters a model stores from the computation it
spends per token \citep{shazeer2017,switch2021,wang2026railwave}, and have become a central
design for recent large language models
\citep{mixtral2024,deepseekv32024,yang2025qwen3,li2026turning}. Sparse activation reduces
computation, but not memory. When GPU memory cannot hold every expert,
inference systems keep expert weights in host memory and cache, prefetch, or
compute on the CPU the experts they need
\citep{moeoffload2023,moeinfinity2024,fiddler2024}, and
host-to-device (H2D) expert transfers enter the critical path of decoding.
Speculative decoding is a natural remedy. A small drafter proposes several
tokens and the target verifies them together in one forward pass
\citep{leviathan2023,chen2023,li2026ets}, so the weights that forward reads serve every
committed token. This is especially attractive when those weights must first be
copied from host memory \citep{specexec2024}.

Existing work improves speculative decoding along two lines. Tree-based
methods draft candidate branches and choose which to verify by their
expected acceptance, ranking branches by draft confidence
\citep{li2024eagle2,li2025eagle3}, maximizing the expected acceptance length
\citep{opttree2024}, or sizing the tree against measured hardware cost
\citep{sequoia2024}. Methods for MoE models act on the
experts, prefetching the experts predicted during drafting
\citep{chen2025spmoe}, limiting the experts loaded during verification
\citep{mcdanel2026moespec}, or training a resident draft expert
\citep{han2026draftexpert}.

Tree selection, however, treats verification cost as a function of how many
tokens are verified, which holds in a dense model, where every token reads the
same weights. An MoE model routes each verified token independently, so a
draft tree activates the union of its tokens' experts, including experts needed
only by branches that will be rejected. A larger tree commits more tokens per
cycle but also transfers more experts, so much of the speedup that its
acceptance length promises is lost. Verification cost depends not only on how many draft tokens are
verified but also on which ones, and a selector that values trees by
acceptance and size cannot see this difference.

We propose SpecFit, which selects the verification subtree that maximizes
predicted throughput rather than acceptance length. It scores each
ancestor-closed subtree of an unmodified EAGLE-3 candidate pool by the ratio
of a draft-probability reward to a cycle-time estimate that grows with the
number of distinct experts the subtree activates in each layer, with
coefficients fitted once on calibration queries. An exact depth-first search
with incremental expert bitmaps then decides both which nodes and how many to
verify. Run after each draft level, the same search also decides when to stop
drafting. The target is untouched: its router is unchanged,
and every routed expert of every selected token is executed.

The design rests on four properties. The
expected committed length of a subtree is one plus the sum of its nodes'
commit probabilities, with no independence assumption; the expert-union cost
is monotone and submodular; a smaller subtree raises throughput if and
only if it saves a larger fraction of cycle time than of committed tokens;
and selecting by the per-cycle ratio stays within a computable margin of the
throughput-optimal rule. On
40 summarization articles and an 80-question GSM8K subset, SpecFit decodes
21.6\% and 7.0\% faster than EAGLE-3 with the same expert cache and
1.97$\times$ and 1.72$\times$ as fast as vanilla EAGLE-3. The gains come from
cheaper verification and drafting rather than longer accepted paths, and hold
on GSM8K, where EAGLE-3 already accepts long draft paths.

In summary, we make the following contributions:
\begin{itemize}
\item We identify and quantify a conflict between speculative decoding and
expert offloading in MoE models: verifying a draft tree enlarges the union of
activated experts, so that verification cost depends on which draft tokens
are verified, not only on how many.
\item We propose SpecFit, which selects the verification subtree and the
drafting depth by the ratio of a draft-probability reward to an expert-union cost and solves this
selection exactly, without retraining the drafter or altering the target's
routing and expert execution.
\item We characterize when a smaller subtree raises throughput, bound the
distance of per-cycle ratio selection from the throughput-optimal rule, and
prove the size-limited selection NP-hard; on summarization and GSM8K,
SpecFit improves decoding throughput over both vanilla EAGLE-3 and EAGLE-3
with an expert cache.
\end{itemize}

\section{Background}
\label{sec:background}

This section introduces the two techniques that SpecFit combines and shows
why, together, they make verification cost depend on which draft tokens are
verified.

\paragraph{MoE inference with expert offloading.}
An MoE model has $L$ MoE layers. Layer $\ell$ contains $N$ expert feed-forward networks and a router
that sends each token $x$ to the set $E_\ell(x)$ of the $k$ experts with the
highest gate scores; the layer output combines those experts' outputs with
the gate weights \citep{shazeer2017,switch2021}. When the experts do not fit in GPU memory, an
offloading runtime keeps them in host memory and uses a limited number of GPU slots per
layer as an expert cache
\citep{moeoffload2023,moeinfinity2024,promoe2024}. A forward pass over a query
$Q$ must execute, in every layer, each expert in $\bigcup_{x\in Q}E_\ell(x)$.
An expert missing from the cache is copied host-to-device before the layer
can finish and evicts another expert. The cache
avoids a transfer when a later forward needs an expert that is still in the
cache, but the per-layer union over the query remains the demand that each
forward must meet.

\paragraph{Tree-based speculative decoding.}
In EAGLE-2 and EAGLE-3 \citep{li2024eagle2,li2025eagle3}, the drafter starts
from the pending root token $s$ and expands its most likely continuations for
several steps into a candidate pool $V$. Each node $v\in V$ has a parent
$\operatorname{pa}(v)\in V\cup\{s\}$ and a path probability $p_v$, the product
of draft conditional probabilities from the root to $v$. A subset
$S\subseteq V$ is then verified: the target processes $S\cup\{s\}$ in one
forward with a tree attention mask and, under greedy verification, accepts the
longest root-to-node path whose tokens match its own predictions. The cycle
commits the root and the accepted path, and the target's next prediction
becomes the new root. EAGLE-3 verifies the nodes with the highest $p_v$ up to
a fixed budget. If a cycle commits $A$ tokens in time $T$, long-run decoding
throughput is $\mathbb{E}[A]/\mathbb{E}[T]$, so choosing $S$ trades committed
tokens against cycle time.

\paragraph{Why the verified set matters.}
Combining the two, write $E_\ell(v)$ for $E_\ell$ of the token of node $v$
in the tree forward, which depends on the node's ancestors as well as on its
own token. A forward that verifies $S\cup\{s\}$
must execute
\begin{equation}
D(S)=\sum_{\ell=1}^{L}\Bigl|E_\ell(s)\cup\bigcup_{v\in S}E_\ell(v)\Bigr|
\label{eq:union}
\end{equation}
experts. Draft tokens at different positions and on different branches are routed
separately, so the union grows with the tree, and a rejected branch pays
for its experts as fully as an accepted one.

Choosing the verified set under this cost is difficult for four reasons. First, the reward is uncertain: acceptance is revealed only by
verification, so selection must rely on estimates available before the target
runs. Second, the cost is not additive: tokens that share an expert in a layer
pay for it once, so the expert union is a coverage function and a node's
marginal cost depends on the rest of the tree. Third, the choices are
structured: a token can be committed only with its prefix, so only
ancestor-closed subsets are meaningful. Fourth, the objective is a ratio:
maximizing acceptance ignores cost, minimizing expert demand is trivially
achieved by verifying no draft token, and the best tree size varies from cycle
to cycle.

The decision also lies on the critical path of every cycle, so it must be
cheap, and it should not buy speed by pruning or substituting routed
experts, which would change the target's computation. Existing selectors value
trees by acceptance and, where they model cost, by size or depth
\citep{li2025eagle3,opttree2024,sequoia2024}; none accounts for the expert
union. Section~\ref{sec:method} formulates and solves this selection.

\section{Method}
\label{sec:method}

In each speculative cycle, SpecFit decides how deep EAGLE-3 drafts and which
draft tokens the target verifies. While the drafter builds the
candidate pool level by level, SpecFit selects the legal subtree with the
largest ratio of predicted reward to predicted cycle time and stops drafting
once that subtree takes no node from the newest level; the target then
verifies the selected subtree and commits an accepted path
(Figure~\ref{fig:pipeline}).
Section~\ref{sec:problem} defines the legal subtrees.
Section~\ref{sec:estimate} predicts a subtree's reward from draft
probabilities and its cycle time from the experts its nodes activate.
Section~\ref{sec:selection} solves the resulting objective exactly, uses it to
decide the drafting depth, relates it to the throughput-optimal rule, and
describes the route information and target execution.

\begin{figure}[t]
\centering
\definecolor{ovblue}{HTML}{0072B2}
\definecolor{ovorange}{HTML}{E69F00}
\definecolor{ovred}{HTML}{D55E00}
\definecolor{ovgreen}{HTML}{009E73}
\definecolor{ovpanel}{HTML}{F4F6F9}
\begin{tikzpicture}[
  x=1cm,y=1cm,>=Stealth,font=\sffamily\scriptsize,
  panel/.style={draw=black!25,fill=ovpanel,rounded corners=3pt},
  stepno/.style={circle,fill=black!80,text=white,font=\sffamily\scriptsize\bfseries,
    inner sep=0pt,minimum size=0.36cm},
  ptitle/.style={anchor=west,font=\sffamily\small\bfseries},
  psub/.style={anchor=west,text=black!65},
  tok/.style={circle,draw=ovblue!80,line width=0.5pt,inner sep=0pt,minimum size=0.25cm},
  unsel/.style={tok,draw=black!35,fill=white,dash pattern=on 1.2pt off 1pt},
  rootn/.style={circle,fill=black!75,text=white,inner sep=0pt,minimum size=0.3cm,
    font=\sffamily\tiny\bfseries},
  tedge/.style={draw=ovblue!70,line width=0.6pt},
  cutedge/.style={draw=black!35,line width=0.5pt,dash pattern=on 1.2pt off 1pt},
  cell/.style={draw=black!30,line width=0.3pt},
  avoid/.style={draw=ovred,line width=0.4pt,dash pattern=on 1pt off 0.8pt},
  flow/.style={->,line width=0.9pt,draw=black!70},
]
\def\ovpool{a/1/3.75/s/92/1,b/1/2.45/s/5/0,c/2/4.15/a/81/1,d/2/3.25/a/9/0,
  j/2/2.45/b/2/0,e/3/4.15/c/66/1,f/3/3.7/c/12/1,h/3/3.25/d/4/0,g/4/4.15/e/49/1,
  k/4/3.7/e/10/0,i/5/4.15/g/35/0}
\def\ovdx{0.48}

\draw[panel] (0,0.95) rectangle (3.3,5.2);
\node[stepno] at (0.3,4.88) {1};
\node[ptitle] at (0.52,4.88) {\enzh{Draft}{草拟}};
\node[psub] at (0.1,4.52) {\enzh{EAGLE-3 candidate pool $V$}{EAGLE-3 候选池 $V$}};
\begin{scope}[shift={(0.42,0)}]
  \node[rootn] (A-s) at (0,3.1) {$s$};
  \foreach \n/\d/\y/\p/\q/\s in \ovpool
    \node[tok,fill=ovblue!\q!white] (A-\n) at (\d*\ovdx,\y) {};
  \foreach \n/\d/\y/\p/\q/\s in \ovpool
    \draw[tedge] (A-\p) -- (A-\n);
\end{scope}
\shade[left color=white,right color=ovblue,draw=black!30,line width=0.3pt]
  (0.7,1.67) rectangle (2.9,1.8);
\node[anchor=east,inner sep=2pt] at (0.7,1.735) {$p_v$};
\node[anchor=north,inner sep=1.5pt] at (0.7,1.67) {0};
\node[anchor=north,inner sep=1.5pt] at (2.9,1.67) {1};

\draw[panel] (3.68,0.95) rectangle (9.55,5.2);
\node[stepno] at (3.98,4.88) {2};
\node[ptitle] at (4.2,4.88) {\enzh{Select}{选择}};
\node[psub] at (3.78,4.52) {\enzh{SpecFit: exact search over legal subtrees}%
  {SpecFit：在合法子树上精确搜索}};
\begin{scope}[shift={(4.08,0)}]
  \node[rootn] (B-s) at (0,3.1) {$s$};
  \foreach \n/\d/\y/\p/\q/\s in \ovpool {
    \ifnum\s=1 \node[tok,fill=ovblue] (B-\n) at (\d*\ovdx,\y) {};
    \else \node[unsel] (B-\n) at (\d*\ovdx,\y) {}; \fi}
  \foreach \n/\d/\y/\p/\q/\s in \ovpool {
    \ifnum\s=1 \draw[tedge] (B-\p) -- (B-\n); \else \draw[cutedge] (B-\p) -- (B-\n); \fi}
\end{scope}
\node[tok,fill=ovblue] at (5.35,2.6) {};
\node[anchor=west] at (5.47,2.6) {\enzh{selected $\widehat S$}{选中 $\widehat S$}};
\node[unsel] at (5.35,2.3) {};
\node[anchor=west] at (5.47,2.3) {\enzh{not verified}{不验证}};
\node[anchor=west] at (6.78,4.12) {\enzh{estimated experts of $\widehat S$}{$\widehat S$ 的估计专家}};
\foreach \row/\lab/\pat in {3.77/{$\ell{=}1$}/{r,r,s,s,o,x,o,s,x,o},
                            3.47/{$\ell{=}2$}/{r,s,o,r,s,x,x,o,s,o},
                            2.95/{$\ell{=}L$}/{s,r,o,s,o,r,x,s,o,x}} {
  \node[anchor=east] at (7.28,\row) {\lab};
  \foreach \t [count=\x] in \pat {
    \pgfmathsetmacro\cx{7.1+0.21*\x}
    \if\t r \fill[black!60] (\cx,\row-0.09) rectangle ++(0.18,0.18);\fi
    \if\t s \fill[ovorange] (\cx,\row-0.09) rectangle ++(0.18,0.18);\fi
    \if\t x \draw[avoid] (\cx+0.01,\row-0.08) rectangle ++(0.16,0.16);\fi
    \if\t o \draw[cell,fill=white] (\cx,\row-0.09) rectangle ++(0.18,0.18);\fi
  }}
\foreach \x in {7.05,8.25} \foreach \y in {3.12,3.21,3.3} \fill[black!70] (\x,\y) circle (0.5pt);
\fill[black!60] (6.88,2.52) rectangle ++(0.16,0.16);
\node[anchor=west] at (7.04,2.6) {\enzh{root}{根节点}};
\fill[ovorange] (\enzh{7.83}{8.02},2.52) rectangle ++(0.16,0.16);
\node[anchor=west] at (\enzh{7.99}{8.18},2.6) {\enzh{added by $\widehat S$}{$\widehat S$ 新增}};
\draw[avoid] (6.89,2.23) rectangle ++(0.14,0.14);
\node[anchor=west] at (7.04,2.3) {\enzh{saved by pruning}{剪枝后省去}};
\draw[draw=black!30,fill=white,rounded corners=2pt] (3.76,1.08) rectangle (9.47,1.98);
\node[font=\small] at (6.615,1.73)
  {$\widehat S=\operatorname*{arg\,max}_{S\in\mathcal A_K}\;\widehat R(S)\,/\,\widehat T(S)$};
\node at (6.615,1.3)
  {$\widehat R(S)=1+\sum_{v\in S}p_v,\quad\widehat T(S)=b_{|S|}+\beta F(S)$};

\draw[panel] (9.93,0.95) rectangle (13.95,5.2);
\node[stepno] at (10.23,4.88) {3};
\node[ptitle] at (10.45,4.88) {\enzh{Verify}{验证}};
\node[psub] at (10.03,4.52) {\enzh{offloaded MoE target}{卸载的 MoE 目标模型}};
\draw[draw=ovgreen!70!black,fill=white,rounded corners=2pt] (10.07,2.72) rectangle (13.81,4.3);
\node[anchor=north west,font=\sffamily\scriptsize\bfseries,text=ovgreen!60!black]
  at (10.1,4.3) {GPU};
\draw[draw=black!30,fill=ovblue!10,rounded corners=1.5pt] (10.2,2.92) rectangle (11.42,3.95);
\node[align=center] at (10.81,3.43) {\enzh{target\\forward\\on $\widehat S$}{目标模型\\前向验证\\$\widehat S$}};
\node[anchor=south] at (12.58,3.72) {\enzh{expert cache}{专家缓存}};
\foreach \t [count=\n from 0] in {h,o,l,h,h,o, l,h,o,h,l,h, h,h,o,l,h,o} {
  \pgfmathtruncatemacro\cx{mod(\n,6)}\pgfmathtruncatemacro\cy{div(\n,6)}
  \pgfmathsetmacro\px{11.74+0.28*\cx}\pgfmathsetmacro\py{3.46-0.27*\cy}
  \if\t h \fill[black!30] (\px,\py) rectangle ++(0.24,0.23);\fi
  \if\t l \fill[ovorange] (\px,\py) rectangle ++(0.24,0.23);\fi
  \if\t o \draw[cell,fill=white] (\px,\py) rectangle ++(0.24,0.23);\fi}
\draw[draw=black!45,fill=white,rounded corners=2pt] (10.07,1.08) rectangle (13.81,2.12);
\node[anchor=north west,font=\sffamily\scriptsize\bfseries,text=black!60]
  at (10.1,2.12) {\enzh{Host memory}{主机内存}};
\node[anchor=north east,text=black!65] at (13.78,2.12) {\enzh{all experts}{全部专家}};
\foreach \n in {0,...,31} {
  \pgfmathtruncatemacro\cx{mod(\n,16)}\pgfmathtruncatemacro\cy{div(\n,16)}
  \pgfmathsetmacro\px{10.23+0.215*\cx}\pgfmathsetmacro\py{1.45-0.23*\cy}
  \draw[cell,fill=black!8] (\px,\py) rectangle ++(0.19,0.19);}
\foreach \n in {3,9,20,26} {
  \pgfmathtruncatemacro\cx{mod(\n,16)}\pgfmathtruncatemacro\cy{div(\n,16)}
  \pgfmathsetmacro\px{10.23+0.215*\cx}\pgfmathsetmacro\py{1.45-0.23*\cy}
  \fill[ovorange] (\px,\py) rectangle ++(0.19,0.19);}
\draw[->,line width=1.6pt,draw=ovred] (12.58,2.14) -- (12.58,2.7);
\node[anchor=east,text=ovred,font=\sffamily\scriptsize\bfseries] at (12.5,2.42)
  {\enzh{H2D over PCIe}{H2D（PCIe）}};
\node[anchor=west,text=black!65] at (12.66,2.42) {\enzh{misses}{未命中}};

\draw[flow] (3.32,3.1) -- (3.66,3.1);
\draw[flow] (9.57,3.45) -- (9.91,3.45);
\draw[flow,rounded corners=4pt] (11.3,0.95) -- (11.3,0.5) -- (1.65,0.5) -- (1.65,0.93);
\node[fill=white,inner xsep=3pt,minimum height=0.4cm] (commit) at (6.615,0.5)
  {\hspace{0.44cm}\enzh{commit the accepted path}{提交接受路径}};
\node[stepno] at ([xshift=0.21cm]commit.west) {4};
\end{tikzpicture}
\caption{One SpecFit decoding cycle. (1)~EAGLE-3 drafts a candidate pool $V$
with path probabilities $p_v$ level by level; drafting stops once the
selection takes no node from the newest level. (2)~SpecFit searches the ancestor-closed
subtrees exactly and selects the one with the largest $\widehat R/\widehat T$;
experts needed only by unselected nodes are excluded from $F(\widehat S)$.
(3)~The offloaded target verifies $\widehat S$. Its own router decides which
experts run, and experts missing from the expert cache are copied from host
memory. (4)~The accepted path is committed.}
\label{fig:pipeline}
\end{figure}
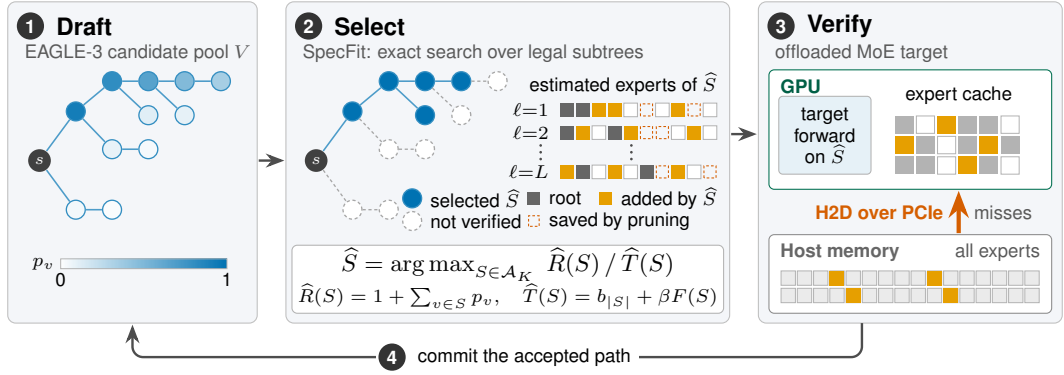

\subsection{Selecting a Verification Subtree}
\label{sec:problem}

With the root $s$, candidate pool $V$, and parents $\operatorname{pa}(v)$ of
Section~\ref{sec:background}, a selected node requires its ancestors, which
defines the legal family of subtrees with at most $K$ draft nodes,
\begin{equation}
\mathcal A_K=\{S\subseteq V:\ |S|\leq K,\quad
 v\in S\Rightarrow\operatorname{pa}(v)\in S\cup\{s\}\}.
\label{eq:feasible}
\end{equation}
Our candidate generator uses top-2 expansion, depth at most 5, and at most
18 nodes, with $K=\min(7,|V|)$; the pool becomes smaller near the generation
limit or when SpecFit stops drafting early (Section~\ref{sec:selection}).

\subsection{Estimating Reward and Expert Cost}
\label{sec:estimate}

A useful subtree must retain likely accepted paths while limiting the work
spent verifying all its nodes. Let $I_v(S)$ indicate that draft node $v$ is
committed when $S$ is verified. By linearity of expectation, the expected
number of tokens committed per cycle is
\begin{equation}
R(S)=1+\sum_{v\in S}\Pr\bigl(I_v(S)=1\bigr)
\;\approx\;1+\sum_{v\in S}p_v=\widehat R(S),
\label{eq:reward}
\end{equation}
in which the root contributes one and each draft node its commit probability.
The first equality holds for any dependence among the commit events, with each
probability allowed to depend on the whole tree. Before verification reveals
the commit events, SpecFit estimates their probabilities with the path
probabilities $p_v$ of Section~\ref{sec:background}, which the drafter already
computes while building the pool.

Let $\widehat E_{\ell,v}$ be the estimate of $E_\ell(v)$ supplied to the
selector for node $v$ in layer $\ell$. Substituting these estimates into
$D$ of Equation~\ref{eq:union} gives the union feature $F$, and SpecFit predicts
cycle time from it and the node count,
\begin{equation}
F(S)=\sum_{\ell=1}^{L}\left|\widehat E_{\ell,s}\cup
     \bigcup_{v\in S}\widehat E_{\ell,v}\right|,
\qquad
\widehat T(S)=b_{|S|}+\beta F(S).
\label{eq:cost}
\end{equation}
The node-count intercepts $b_0,\ldots,b_7$ and the shared coefficient $\beta$
are fitted once by nonnegative least squares to cycle times measured on
GSM8K calibration questions, then frozen (Appendix~\ref{app:calibration}).
The fitted cycle
time covers drafting, target verification, and commit; the selector's own time
is measured and included in the reported throughput.

$F(S)$ estimates $D(S)$, the expert demand that the expert cache serves;
every selected branch contributes to it, whether or not it is later accepted. SpecFit with the expert cache uses
coefficients fitted with the cache active, so $\beta$ is the average time per
expert of demand with cache hits folded in; without the cache, it uses
coefficients fitted with the cache cleared before every target forward. Either
way, $\widehat T$ needs only the routes of the pool.
For fixed estimates, $F$ is a coverage function. Adding $v$ to $S$ increases
it by $\sum_\ell|\widehat E_{\ell,v}\setminus\widehat U_\ell(S)|$, where
$\widehat U_\ell(S)$ is the layer-$\ell$ union in Equation~\ref{eq:cost}; since
$\widehat U_\ell(S)$ only grows with $S$, $F$ is monotone and submodular.
Each node adds only experts new to the subtree, so its marginal expert
demand can only fall as the subtree grows.

\subsection{Selection, Verification, and Available Information}
\label{sec:selection}

SpecFit selects
\begin{equation}
\widehat S=\operatorname*{arg\,max}_{S\in\mathcal A_K}
             \frac{\widehat R(S)}{\widehat T(S)}.
\label{eq:objective}
\end{equation}
This objective jointly chooses node identities and cardinality. For positive
rewards and times, a subtree with $(R,T)$ outperforms one with $(R_0,T_0)$
exactly when
\begin{equation}
\frac{R}{T}>\frac{R_0}{T_0}
\quad\Longleftrightarrow\quad
\frac{T_0-T}{T_0}>\frac{R_0-R}{R_0}.
\label{eq:gain}
\end{equation}
Pruning nodes therefore raises throughput if and only if it saves a larger
fraction of cycle time than of committed tokens.

Equation~\ref{eq:objective} maximizes the ratio within each cycle, whereas
throughput is a ratio of expectations over cycles
(Section~\ref{sec:background}). To relate the two, let the cycle states, each
a pool with its probabilities and routes, be independent and identically
distributed and independent of earlier selections, and let
$\Theta(\pi)=\mathbb E[\widehat R(S_\pi)]/\mathbb E[\widehat T(S_\pi)]$ be
the predicted throughput of a rule $\pi$ that selects $S_\pi$ in each cycle.
The throughput-optimal rule selects
$S^\ast=\operatorname{arg\,max}_{S\in\mathcal A_K}\widehat R(S)-\lambda^\ast\widehat T(S)$
and attains the largest value $\lambda^\ast=\max_\pi\Theta(\pi)$
\citep{dinkelbach1967}. This rule needs $\lambda^\ast$, which depends on the
distribution of all cycles; Equation~\ref{eq:objective} needs only the
current pool.

\begin{restatable}{proposition}{restatethroughput}
\label{prop:throughput}
Let $\rho=\widehat R(\widehat S)/\widehat T(\widehat S)$ be the largest ratio
in a cycle, and let $\Theta$ be the predicted throughput of
Equation~\ref{eq:objective}. Then
\begin{equation}
0\le\lambda^\ast-\Theta\le
\frac{\mathbb E\bigl[(\rho-\lambda^\ast)\bigl(\widehat T(S^\ast)-\widehat T(\widehat S)\bigr)\bigr]}
     {\mathbb E\bigl[\widehat T(\widehat S)\bigr]},
\label{eq:margin}
\end{equation}
and the quantity inside the expectation is nonnegative in every cycle.
\end{restatable}

The margin vanishes when every cycle has the same largest ratio, and it is
large only when cycles differ widely in their largest ratio and the two rules
then select trees of widely different time. Appendix~\ref{app:theory} proves
the proposition and evaluates the margin on recorded pools. The same argument
shows that, for fixed intercepts, a larger $\beta$ can only lower the selected
tree's predicted expert demand per committed token,
$F(\widehat S)/\widehat R(\widehat S)$.

For general pools, the maximization in Equation~\ref{eq:objective} is hard.

\begin{restatable}{proposition}{restatehardness}
\label{prop:hardness}
Deciding whether some $S\in\mathcal A_K$ reaches a given ratio
$\widehat R(S)/\widehat T(S)$ is NP-complete, even for a pool of depth one,
one MoE layer, and equal intercepts. Without the size limit $K$ and with
intercepts affine and nondecreasing in $|S|$, the largest ratio is computable
in polynomial time by parametric minimum cut \citep{picard1976,gallo1989}.
\end{restatable}

SpecFit's pools have at most 18 nodes and $K\le7$, so the legal family is
small, and SpecFit solves Equation~\ref{eq:objective} exactly by enumeration
(Algorithm~\ref{alg:selection} in Appendix~\ref{app:theory}). A depth-first
traversal reaches every subtree in $\mathcal A_K$ exactly once and carries the
per-layer expert unions as bitmaps from parent states. Equal ratios prefer
larger reward, smaller union, fewer draft nodes, and then lexicographically
smaller node IDs, fixing a unique choice among the maximizers of
Equation~\ref{eq:objective} over $\mathcal A_K$.

\paragraph{Adaptive drafting depth.}
The same objective also decides how deep EAGLE-3 drafts. The drafter adds one
level of the pool per draft forward; let $V_{\le d}$ be the nodes of the first
$d$ levels. After drafting level $d$, SpecFit solves
Equation~\ref{eq:objective} with $V_{\le d}$ in place of $V$. If the selected
subtree contains a level-$d$ node, SpecFit drafts the next level; otherwise it
stops drafting and verifies that subtree. At the maximum depth it selects over
the full pool. The stop rests on the tree structure: a deeper node can be
selected only together with its level-$d$ ancestor and adds no more reward
than that ancestor, which the selection has already left out. The stopped
selection equals the full-pool selection whenever the latter stays within the
first $d$ levels (Lemma~\ref{lem:depth} in Appendix~\ref{app:theory}).

\paragraph{Route information.}
In our experiments, $\widehat E$ comes from a route atlas, a separate target
query over the full candidate pool on SpecFit's reference prefix, prepared
outside the request timer. During measurement, each cycle's prefix and pool
are checked against the atlas. Selection uses only the draft probabilities,
the atlas routes, and the frozen coefficients.

\paragraph{Target execution.}
After selection, the target verifies $\widehat S$ with its own router and
executes every routed expert of every selected token; experts missing from the
expert cache are copied from host memory.

\section{Experiments}
\label{sec:experiments}

We compare SpecFit with EAGLE-3 and autoregressive (AR) decoding, each without
and with the
expert cache, on summarization and GSM8K, decompose the gain with verification
and transfer counters, vary the cache capacity and the GPU, and account for
route prediction and selection cost.

\subsection{Setup}
\label{sec:setup}

The target is Qwen3-30B-A3B-Instruct-2507 \citep{yang2025qwen3,qwen2025model}
with its RedHatAI EAGLE-3 drafter \citep{redhat2025drafter}, and the workloads
are 40 CNN/DailyMail articles \citep{nallapati2016} and 80 GSM8K questions
\citep{cobbe2021} from Spec-Bench \citep{xia2024}. All measurements run at
request batch size one, on a single A100 PCIe 40GB GPU except in
Section~\ref{sec:device}. Every configuration, including AR, runs on our
offload backend, which keeps all expert weights in pinned host memory. With
the expert cache, each layer holds 48 experts in GPU slots with LRU
replacement; without it, the cache is cleared before every target forward.
Table~\ref{tab:methods} lists the compared methods; SpecFit and EAGLE-3 use
the same drafter and differ only in drafting depth and in which and how many
nodes are verified.
EAGLE-3 without the cache is the vanilla EAGLE-3 of the introduction; the
cached pair is our primary comparison, and SpecFit without qualification
denotes its cached configuration. The cost coefficients are fitted on GSM8K
calibration questions, with and without the cache, and used unchanged on both
workloads. Speedups are relative to AR without the expert cache, and committed
tokens are the tokens a verification cycle appends to the output, namely the
root and the accepted draft nodes (Section~\ref{sec:background}).
Appendix~\ref{app:calibration} gives the backend, timing, calibration, and data
details.

\begin{table}[htbp]
\centering
\caption{Compared methods.}
\label{tab:methods}
\begin{tabular}{ll}
\toprule
Method & Verification-node selection \\
\midrule
AR & Serial target decoding, without a drafter \\
EAGLE-3 & Highest cumulative probability, at most 7 nodes \\
SpecFit & Equation~\ref{eq:objective} with adaptive depth, at most 7 nodes \\
\bottomrule
\end{tabular}
\end{table}

\subsection{Summarization}
\label{sec:summarization}

To test whether SpecFit gains by leaving out draft nodes that do not pay for
their experts, we first use summarization, on which EAGLE-3 accepts short
draft paths. Each method runs on the 40 CNN/DailyMail articles without and
with the expert cache; Figure~\ref{fig:summ} reports decode throughput.

\begin{figure}[htbp]
\centering
\input{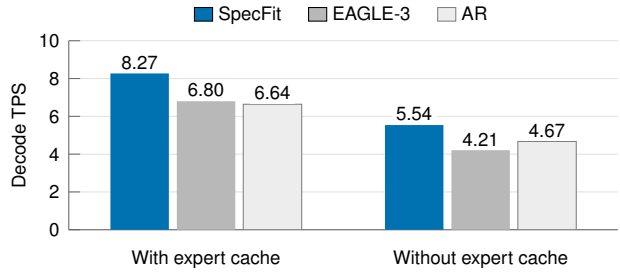}
\begin{tikzpicture}
\begin{axis}[tpsbar, width=0.52\linewidth, xmin=-0.5, xmax=1.5, xtick={0,1},
  xticklabels={{\enzh{With expert cache}{使用专家缓存}}, {\enzh{Without expert cache}{不使用专家缓存}}},
  ymax=10, ytick={0,2,4,6,8,10}]
\addplot[specfit] coordinates {(0,8.27) (1,5.54)};
\addplot[eagle] coordinates {(0,6.80) (1,4.21)};
\addplot[arbar] coordinates {(0,6.64) (1,4.67)};
\legend{SpecFit, EAGLE-3, AR}
\end{axis}
\end{tikzpicture}
\caption{Results on the CNN/DailyMail summarization dataset.}
\label{fig:summ}
\end{figure}

SpecFit decodes 8.27 tokens/s, 21.6\% above cached EAGLE-3 (19.9\% including
prefill) and 1.77$\times$ AR. Without the cache, it decodes 31.6\% faster than
vanilla EAGLE-3, which is itself slower than AR (0.90$\times$). The
verification and transfer counters in Table~\ref{tab:summ-mechanism} show
where the gain comes from. EAGLE-3 verifies seven draft nodes per cycle but
accepts only 1.54 of them. SpecFit verifies 2.88 draft nodes instead. It
stops drafting before the fifth level in 85.8\% of cycles, which cuts
drafting time per cycle by 44.9\%. It commits 13.1\% fewer tokens per cycle
than cached EAGLE-3, but each cycle takes 28.5\% less decode time, so it
decodes 21.6\% faster. Verification accounts for 78.3\% of the time saved
per cycle, and SpecFit transfers 22.1\% fewer experts per output token. This
is the trade the objective makes: a node with a low acceptance probability
adds little to the expected commit but still adds its experts to the
verified union, so leaving such nodes out shortens the cycle by a larger
fraction than it reduces committed tokens.

\begin{table}[htbp]
\centering
\small
\caption{Verification and transfer counters on the CNN/DailyMail summarization dataset.}
\label{tab:summ-mechanism}
\setlength{\tabcolsep}{4pt}
\begin{tabular}{lcccccc}
\toprule
Method & Cache & Draft nodes & Committed & Cycles & ms/cycle & H2D/token \\
\midrule
AR & No & -- & $1.00$ & 9,652 & 214.2 & 400.7 \\
EAGLE-3 & No & $7.00$ & $2.54$ & 3,823 & 600.9 & 486.5 \\
SpecFit & No & $2.29$ & $2.08$ & 4,702 & 373.6 & 347.7 \\
\midrule
AR & Yes & -- & $1.00$ & 9,652 & 150.6 & 49.7 \\
EAGLE-3 & Yes & $7.00$ & $2.54$ & 3,823 & 371.7 & 73.4 \\
SpecFit & Yes & $2.88$ & $2.21$ & 4,389 & 265.8 & 57.2 \\
\bottomrule
\end{tabular}
\end{table}

\subsection{GSM8K}
\label{sec:gsm8k}

To test whether SpecFit still gains when EAGLE-3 accepts long draft paths,
we run the same comparison on the 80 GSM8K questions;
Figure~\ref{fig:gsm8k} reports decode throughput.

\begin{figure}[htbp]
\centering
\input{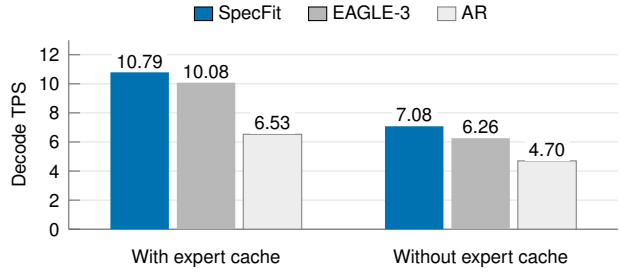}
\begin{tikzpicture}
\begin{axis}[tpsbar, width=0.52\linewidth, xmin=-0.5, xmax=1.5, xtick={0,1},
  xticklabels={{\enzh{With expert cache}{使用专家缓存}}, {\enzh{Without expert cache}{不使用专家缓存}}},
  ymax=13, ytick={0,2,4,6,8,10,12}]
\addplot[specfit] coordinates {(0,10.79) (1,7.08)};
\addplot[eagle] coordinates {(0,10.08) (1,6.26)};
\addplot[arbar] coordinates {(0,6.53) (1,4.70)};
\legend{SpecFit, EAGLE-3, AR}
\end{axis}
\end{tikzpicture}
\caption{Results on the GSM8K dataset.}
\label{fig:gsm8k}
\end{figure}

SpecFit decodes 10.79 tokens/s, 7.0\% above cached EAGLE-3 (6.3\% including
prefill) and 2.29$\times$ AR; without the cache, it decodes 13.1\% faster than
vanilla EAGLE-3. The verification and transfer counters in
Table~\ref{tab:gsm8k-mechanism} show where this gain comes from. With the
cache, EAGLE-3 commits 4.35 tokens per cycle,
against 2.54 on summarization, so fewer of its nodes are worth leaving out.
SpecFit still verifies 5.09 draft nodes instead of 7.00 and commits 7.8\%
fewer tokens per cycle, but each cycle takes 13.8\% less decode time, and it
transfers 11.6\% fewer experts per output token; adaptive depth stops
drafting before the fifth level in 33.7\% of cycles and cuts drafting time
per cycle by 11.7\%. The gain is smaller than on summarization but has the
same source, shorter cycles rather than longer accepted paths: cycle time
falls by a larger fraction than committed tokens, as Equation~\ref{eq:gain}
requires.

\begin{table}[htbp]
\centering
\small
\caption{Verification and transfer counters on the GSM8K dataset.}
\label{tab:gsm8k-mechanism}
\setlength{\tabcolsep}{4pt}
\begin{tabular}{lcccccc}
\toprule
Method & Cache & Draft nodes & Committed & Cycles & ms/cycle & H2D/token \\
\midrule
AR & No & -- & $1.00$ & 17,791 & 212.5 & 398.2 \\
EAGLE-3 & No & $7.00$ & $4.35$ & 4,097 & 692.3 & 345.7 \\
SpecFit & No & $4.27$ & $3.81$ & 4,662 & 536.7 & 294.1 \\
\midrule
AR & Yes & -- & $1.00$ & 17,791 & 153.1 & 76.5 \\
EAGLE-3 & Yes & $7.00$ & $4.35$ & 4,097 & 429.8 & 92.7 \\
SpecFit & Yes & $5.09$ & $4.01$ & 4,407 & 370.4 & 82.0 \\
\bottomrule
\end{tabular}
\end{table}

\subsection{Sensitivity to Offload Pressure}
\label{sec:pressure}

To test whether SpecFit's gain grows as the expert cache shrinks, we rerun
EAGLE-3 and SpecFit on the 40 summarization articles with 24 and 12 expert
slots per layer and compare them with the 48 slots of the main runs and with
the configurations without the cache (0 slots); Figure~\ref{fig:slots}
reports decode throughput.

\begin{figure}[htbp]
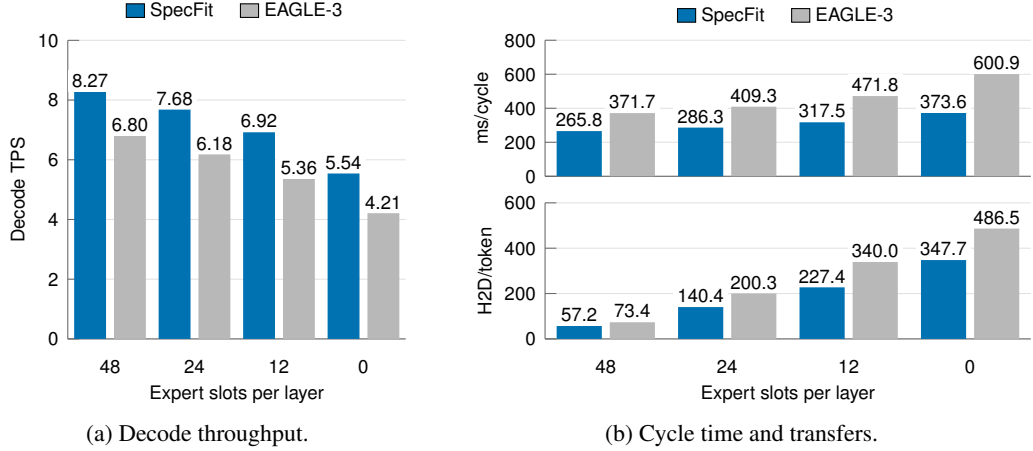

\centering
\subfloat[Decode throughput.]{
\input{fig-bars.tex}
\begin{tikzpicture}
\begin{axis}[tpsbar, bar width=12pt, width=0.32\textwidth, height=3.95cm, xmin=-0.5, xmax=3.5,
  xtick={0,1,2,3}, xticklabels={48,24,12,0},
  xlabel={\enzh{Expert slots per layer}{每层专家槽位数}}, xlabel style={yshift=2pt},
  ymax=10, ytick={0,2,4,6,8,10}]
\addplot[specfit] coordinates {(0,8.27) (1,7.68) (2,6.92) (3,5.54)};
\addplot[eagle] coordinates {(0,6.80) (1,6.18) (2,5.36) (3,4.21)};
\legend{SpecFit, EAGLE-3}
\end{axis}
\end{tikzpicture}\label{fig:slots}}\hfill
\subfloat[Cycle time and transfers.]{
\input{fig-bars.tex}
\begin{tikzpicture}
\pgfplotsset{slotcost/.style={tpsbar, bar width=17pt, width=0.46\textwidth, height=1.8cm,
  xmin=-0.5, xmax=3.5, xtick={0,1,2,3},
  every node near coord/.append style={/pgf/number format/precision=1}}}
\begin{axis}[slotcost, name=cycle, xticklabels={},
  ylabel={\enzh{ms/cycle}{毫秒/周期}}, ymax=800, ytick={0,200,400,600,800}]
\addplot[specfit] coordinates {(0,265.8) (1,286.3) (2,317.5) (3,373.6)};
\addplot[eagle] coordinates {(0,371.7) (1,409.3) (2,471.8) (3,600.9)};
\legend{SpecFit, EAGLE-3}
\end{axis}
\begin{axis}[slotcost, at={(cycle.south west)}, anchor=north west, yshift=-10pt,
  xticklabels={48,24,12,0},
  xlabel={\enzh{Expert slots per layer}{每层专家槽位数}}, xlabel style={yshift=2pt},
  ylabel={H2D/token}, ymax=600, ytick={0,200,400,600}]
\addplot[specfit] coordinates {(0,57.2) (1,140.4) (2,227.4) (3,347.7)};
\addplot[eagle] coordinates {(0,73.4) (1,200.3) (2,340.0) (3,486.5)};
\end{axis}
\end{tikzpicture}\label{fig:slots-mechanism}}
\caption{Results on CNN/DailyMail with different numbers of expert slots per
layer.}
\end{figure}

SpecFit's decoding gain over EAGLE-3 grows from 21.6\% at 48 slots to
24.3\% and 29.2\% at 24 and 12 slots, and it is 31.6\% without the cache. With
the cache, SpecFit verifies 2.88 draft nodes and commits 13.1\% fewer tokens
per cycle than EAGLE-3 at every capacity, so the growing gain comes from
shorter cycles. As Figure~\ref{fig:slots-mechanism} shows, at 48, 24, and 12
slots, its time per cycle is 28.5\%, 30.1\%, and 32.7\% shorter, and it
transfers 22.1\%, 29.9\%, and 33.1\% fewer experts per output token.

Because the cache capacity does not enter selection, the trees at 48, 24,
and 12 slots are the same; for these capacities, let a configuration spend
decode time $T=B+\tau H$ on $R$ decode tokens, a fixed part $B$ plus $H$ expert
transfers of time $\tau$ each. For SpecFit and EAGLE-3 (subscripts S and E),
\begin{equation}
\frac{R_S/T_S}{R_E/T_E}=(1-w)\,\frac{R_S/B_S}{R_E/B_E}+w\,\frac{R_S/H_S}{R_E/H_E},
\qquad w=\frac{\tau H_S}{T_S},
\label{eq:mediant}
\end{equation}
so the gain averages SpecFit's advantages in tokens per unit of fixed time and
per decode-stage transfer, weighted by its transfer-time share $w$. If $B$ and
$\tau$ do not depend on capacity, a smaller cache raises $w$, increasing the
weight of the second term. The full-request counters, including prefill, show
token-per-transfer advantages of 28.3\%, 42.7\%, and 49.5\% at 48, 24, and
12 slots (Figure~\ref{fig:slots-mechanism}), above each measured decoding gain.

\subsection{Another GPU}
\label{sec:device}

To test whether the gain carries over to a GPU with less memory and a slower
host link, we repeat the GSM8K comparison on a single RTX 4090 24GB. It holds
32 expert slots per layer instead of 48, and with the cache cleared, the
fitted time per loaded expert is 0.78\,ms, against 0.41\,ms on the A100. The
cost coefficients are refitted on this GPU with the same protocol
(Appendix~\ref{app:calibration}), and the comparison uses the 40
odd-numbered of the 80 questions; Table~\ref{tab:device} reports the
results.

\begin{table}[htbp]
\centering
\small
\caption{Results on the GSM8K dataset on a single RTX 4090 24GB.}
\label{tab:device}
\setlength{\tabcolsep}{4pt}
\begin{tabular}{lcccccccc}
\toprule
Method & Cache & Decode TPS & Speedup & Draft nodes & Committed & Cycles & ms/cycle & H2D/token \\
\midrule
AR & No & $2.67$ & $1.00\times$ & -- & $1.00$ & 8,714 & 374.4 & 398.6 \\
EAGLE-3 & No & $3.29$ & $1.23\times$ & $6.99$ & $4.32$ & 2,009 & 1,306.7 & 346.4 \\
SpecFit & No & $3.89$ & $1.46\times$ & $3.78$ & $3.68$ & 2,339 & 942.8 & 288.0 \\
\midrule
AR & Yes & $4.24$ & $1.59\times$ & -- & $1.00$ & 8,714 & 236.1 & 126.9 \\
EAGLE-3 & Yes & $6.05$ & $2.27\times$ & $6.99$ & $4.32$ & 2,009 & 711.1 & 150.0 \\
SpecFit & Yes & $6.70$ & $2.51\times$ & $4.53$ & $4.00$ & 2,162 & 593.7 & 128.3 \\
\bottomrule
\end{tabular}
\end{table}

With the expert cache, SpecFit decodes 6.70 tokens/s on this GPU, 10.8\%
above cached EAGLE-3 and 2.51$\times$ AR; without the cache, it decodes
18.1\% faster than vanilla EAGLE-3. As on the A100, it verifies fewer
draft nodes (4.53 instead of 6.99) and commits 7.5\% fewer tokens per cycle,
but each cycle takes 16.5\% less decode time, and it transfers 14.5\% fewer
experts per output token. SpecFit thus makes the same trade on a GPU whose
memory and transfer costs differ.

\subsection{Route Prediction and Selection Cost}
\label{sec:overhead}

Our measurements take routes from the atlas, outside the request timer, so that
selection is evaluated on its own; a deployment replaces the atlas with a route
predictor. Lightweight expert predictors reach 84.7\% average accuracy
\citep{promoe2024} and 94.7\% on Qwen3-30B-A3B \citep{zhu2025preattention}.
SpecFit reads only the predicted expert sets $\widehat E$, and verification
still executes every expert that the target's router chooses, so route errors
can change only which nodes are selected. With a simulated predictor of 60\%
accuracy, which replaces each routed expert by a random expert outside the
true top-8 with probability 40\%, SpecFit still decodes 2.26$\times$ as fast
as AR on 40 GSM8K questions, against 2.27$\times$ with exact routes, because
the correct predictions still show which draft tokens share experts, while
wrong ones rarely coincide across nodes and add a nearly uniform cost per node.
A linear-head predictor, which maps each draft node's hidden state to logits
over all $48\times128$ experts, has 12.6M parameters, the memory of about three
expert slots, and needs about 25\,MFLOP per node; applied after each drafted
level, it adds an estimated 0.1--0.2\,ms per cycle. SpecFit's own selection,
included in every measurement above, takes a median of 2.82\,ms per cycle on
GSM8K (0.7\% of decode time) and 1.43\,ms on summarization (0.6\%), both small
compared with the 59\,ms by which a SpecFit cycle is shorter than a cached
EAGLE-3 cycle on GSM8K. Training route predictors is left to future work.

\section{Related Work}
\label{sec:related}

\paragraph{Expert offloading.}
Offloading systems keep MoE expert weights in host memory and cache and
prefetch the experts they expect to need
\citep{moeoffload2023,moeinfinity2024,promoe2024}, or compute some experts on
the CPU \citep{fiddler2024}. SpecFit runs on an on-demand
expert cache without prefetching; it changes which tokens create expert
demand, not how it is served.

\paragraph{Draft-tree construction and selection.}
\citet{leviathan2023} and \citet{chen2023} develop draft-and-verify
acceleration with target-distribution-preserving sampling rules under their
assumptions. EAGLE-2 constructs context-dependent trees ranked by draft
confidence, and EAGLE-3 develops the draft model and feature interface
\citep{li2024eagle2,li2025eagle3}. OPT-Tree maximizes the expected acceptance
length \citep{opttree2024}, Sequoia chooses tree size and depth against
measured hardware cost \citep{sequoia2024}, and SpecExec verifies large trees
to amortize weight loading for dense models with offloaded weights
\citep{specexec2024}. SpecFit keeps a fixed EAGLE-3 drafter and candidate generator
and changes only how deep it drafts and which subset is sent to verification,
valuing both by an expert-union cost rather than by size or shape.

\paragraph{Speculative decoding for MoE models.}
SP-MoE prefetches experts predicted during drafting and pipelines their
execution \citep{chen2025spmoe}. MoE-Spec limits the experts loaded during
verification and thereby alters the target computation
\citep{mcdanel2026moespec}, whereas SpecFit executes all routed experts for
its selected tokens. DraftExpert combines a trained resident draft expert
with expansion-aware truncation and prefetching for offloaded inference
\citep{han2026draftexpert}. SpecFit instead evaluates a complete-subtree
reward/cost ratio over a bounded feasible family, with cardinality included
in the search, and trains no model component.

\section{Conclusion}

Under expert offloading, speculative verification in MoE models pays for every
expert that the verified tree activates. SpecFit therefore selects each
cycle's verification subtree and drafting depth by the ratio of a draft-probability reward to an
expert-union cost, solved exactly over ancestor-closed subtrees, without
changing the drafter, the router, or expert execution. With Qwen3-30B-A3B, it
verifies fewer draft nodes when their expert cost outweighs their expected
tokens. On summarization, where EAGLE-3 accepts short draft paths, this raises
decoding throughput by 21.6\% over EAGLE-3 with the same expert cache. On
GSM8K, where EAGLE-3 accepts long paths, each cycle is still 13.8\% shorter
and the gain is 7.0\%, to 2.29$\times$ that of AR decoding without the cache.
As the expert cache shrinks from 48 to 12 slots per layer, the summarization
gain grows to 29.2\%, and on an RTX 4090, whose host link is slower, the
GSM8K gain is 10.8\%.

\section*{AI use statement}
We used AI assistants for software development, experiment inspection,
mathematical exposition, literature search, and drafting of the manuscript.
The authors reviewed all AI-assisted content and take full responsibility for
the claims, references, and results in this paper.

\section*{Ethics statement}
This work speeds up decoding of existing MoE language models whose experts
are offloaded to host memory. It uses a publicly released target model,
drafter, and datasets, involves no human subjects, and collects no new data.
SpecFit changes only which draft tokens are sent to verification, and the
target model verifies them unmodified, so the method adds no capability or
risk beyond those of the model it accelerates.

\section*{Reproducibility statement}
Section~\ref{sec:setup} and Appendix~\ref{app:calibration} give the model,
the workloads, the expert cache configuration, the calibration protocol, and
the fitted cost coefficients. Section~\ref{sec:method} and
Algorithm~\ref{alg:selection} specify the selection and the drafting-depth
rule, and Appendix~\ref{app:theory} states the assumptions and gives complete
proofs of both propositions and both lemmas. The supplementary material
contains the source code, with the configuration files and scripts that run
the calibration and the comparisons in Section~\ref{sec:experiments}.

\bibliography{references}
\bibliographystyle{iclr2027_conference}

\newpage

\appendix

\section{Selection Theory and Proofs}
\label{app:theory}

\subsection{Throughput of a Selection Rule}

A cycle state comprises the pool with its probabilities, routes, and legal
family $\mathcal A_K$, and a rule $\pi$ selects $S_\pi\in\mathcal A_K$ from the
current state. Suppose that the states of successive cycles are independent
and identically distributed and independent of earlier selections. If
$\widehat R$ and $\widehat T$ equal the expected committed tokens and cycle
time given the state, the renewal--reward theorem gives the long-run
throughput $\Theta(\pi)=\mathbb E[\widehat R(S_\pi)]/\mathbb E[\widehat T(S_\pi)]$;
otherwise $\Theta(\pi)$ is the throughput that the proxies predict. Let
\begin{equation}
h(\lambda)=\mathbb E\Bigl[\max_{S\in\mathcal A_K}\bigl(\widehat R(S)-\lambda\widehat T(S)\bigr)\Bigr].
\label{eq:parametric}
\end{equation}
The intercepts are positive, so $\widehat T$ is bounded below by a positive
constant, and $h$ is continuous and strictly decreasing, with $h(0)\ge1$ and
$h(\lambda)\to-\infty$; let $\lambda^\ast$ be its root. Every rule satisfies
$\mathbb E[\widehat R(S_\pi)-\lambda^\ast\widehat T(S_\pi)]\le h(\lambda^\ast)=0$,
so $\Theta(\pi)\le\lambda^\ast$, and the rule selecting $S^\ast$ attains
equality. Hence $\lambda^\ast=\max_\pi\Theta(\pi)$ \citep{dinkelbach1967}.

\begin{lemma}
\label{lem:monotone}
In any cycle, let $S_\lambda$ maximize $\widehat R(S)-\lambda\widehat T(S)$
over $\mathcal A_K$. Then
$(\rho-\lambda)\bigl(\widehat T(S_\lambda)-\widehat T(\widehat S)\bigr)\ge0$.
In particular, $\widehat T(S^\ast)\le\widehat T(\widehat S)$ when
$\rho<\lambda^\ast$, and $\widehat T(S^\ast)\ge\widehat T(\widehat S)$ when
$\rho>\lambda^\ast$.
\end{lemma}

\begin{proof}
Since $\rho$ is the largest ratio, $\widehat R(S)-\rho\widehat T(S)\le0$ for
every $S$, with equality at $\widehat S$, so $\widehat S$ maximizes
$\widehat R-\rho\widehat T$. Adding
$\widehat R(S_\lambda)-\lambda\widehat T(S_\lambda)\ge\widehat R(\widehat S)-\lambda\widehat T(\widehat S)$
and
$\widehat R(\widehat S)-\rho\widehat T(\widehat S)\ge\widehat R(S_\lambda)-\rho\widehat T(S_\lambda)$
gives the claim.
\end{proof}

\restatethroughput*

\begin{proof}
The lower bound holds because $\lambda^\ast$ is the largest predicted
throughput. Since $\widehat R(\widehat S)=\rho\widehat T(\widehat S)$,
$\lambda^\ast-\Theta=-\mathbb E[(\rho-\lambda^\ast)\widehat T(\widehat S)]/\mathbb E[\widehat T(\widehat S)]$.
Since $\widehat R(S^\ast)\le\rho\widehat T(S^\ast)$,
$0=h(\lambda^\ast)=\mathbb E[\widehat R(S^\ast)-\lambda^\ast\widehat T(S^\ast)]\le\mathbb E[(\rho-\lambda^\ast)\widehat T(S^\ast)]$.
Adding this nonnegative quantity to the numerator gives
Equation~\ref{eq:margin}, and Lemma~\ref{lem:monotone} with
$\lambda=\lambda^\ast$ gives the sign of each term. If $\rho$ is the same in
every cycle, $\widehat R(S)\le\rho\widehat T(S)$ holds for every selection,
so $\lambda^\ast=\rho=\Theta$.
\end{proof}

\paragraph{Expert demand and $\beta$.}
Fix the intercepts, and let $\widehat S_1$ and $\widehat S_2$ be the
selections under $\beta_1<\beta_2$. Write the inverse ratio as
$\widehat T(S)/\widehat R(S)=x(S)+\beta y(S)$ with
$x(S)=b_{|S|}/\widehat R(S)$ and $y(S)=F(S)/\widehat R(S)$. Each selection
minimizes the inverse ratio under its own coefficient, so
$x(\widehat S_1)+\beta_1y(\widehat S_1)\le x(\widehat S_2)+\beta_1y(\widehat S_2)$
and
$x(\widehat S_2)+\beta_2y(\widehat S_2)\le x(\widehat S_1)+\beta_2y(\widehat S_1)$.
Adding them gives $(\beta_2-\beta_1)\bigl(y(\widehat S_2)-y(\widehat S_1)\bigr)\le0$,
so $F(\widehat S_2)/\widehat R(\widehat S_2)\le F(\widehat S_1)/\widehat R(\widehat S_1)$.

\subsection{Adaptive Depth}

\begin{lemma}
\label{lem:depth}
Let $\widehat S_d$ be the selection over $V_{\le d}$, with $\widehat S_0$ the
root-only choice. The legal families are nested in $d$, so the largest ratio
is nondecreasing in $d$. Adaptive depth stops at the first level $d$ with
$\widehat S_d=\widehat S_{d-1}$, and its selection equals the full-pool
selection $\widehat S$ unless $\widehat S$ contains a node deeper than $d$.
\end{lemma}

\begin{proof}
A legal subtree of $V_{\le d}$ has at most $\min(7,|V_{\le d}|)$ nodes and is
therefore a legal subtree of $V_{\le d'}$ for every $d'>d$, which gives the
nesting and the monotone largest ratio. The tie rule orders all subtrees
independently of the pool, so a subtree that is optimal in a family is optimal
in every subfamily containing it. If $\widehat S_d$ has no level-$d$ node, it
lies in the family of $V_{\le d-1}$ and hence equals $\widehat S_{d-1}$;
conversely, $\widehat S_{d-1}$ has no level-$d$ node. Likewise, if
$\widehat S\subseteq V_{\le d}$, then $\widehat S_d=\widehat S$.
\end{proof}

In cycles whose largest ratio is below $\lambda^\ast$, the throughput-optimal
rule selects no more predicted time than Equation~\ref{eq:objective}
(Lemma~\ref{lem:monotone}), so a stopped tree with less predicted time than
$\widehat S$ departs from Equation~\ref{eq:objective} in the same direction as
that rule.

\subsection{Complexity and Exact Search}

\restatehardness*

\begin{proof}
A subtree is a certificate checkable in polynomial time, so the problem is in
NP. We reduce from CLIQUE \citep{karp1972}: given a graph with $n$ vertices,
$m$ edges, and an integer $2\le r\le n$, decide whether some $r$ vertices are
pairwise adjacent. Let $k=r(r-1)/2$ and $p=1/(m+1)$. The pool has one node per
edge, each a child of the root. The single MoE layer has one expert per
vertex; the node of edge $\{u,w\}$ routes to $\{u,w\}$, and the root routes to
no expert. Set $p_v=p$ for every node, $b_j=b>0$ for every $j$, $K=k$,
$0<\beta<bp/\bigl(n(1+(k-1)p)\bigr)$, and threshold
$\theta=(1+kp)/(b+\beta r)$. Then $F(S)$ is the number of vertices covered by
the edges in $S$, and the edges of an $r$-clique reach $\theta$. Conversely, a
subtree with fewer than $k$ nodes has ratio at most $(1+(k-1)p)/b<\theta$ by
the choice of $\beta$, so a subtree reaching $\theta$ has $k$ edges covering
at most $r$ vertices. Since $q$ vertices span at most $q(q-1)/2$ edges, these
edges cover exactly $r$ vertices and are all pairs among them, an $r$-clique.

Without the size limit, let $b_j=b_0+aj$ with $b_0>0$ and $a\ge0$. Apart from
the root's experts, $F(S)$ counts the pairs $(\ell,e)$ such that some $v\in S$
routes to expert $e$ in layer $\ell$. For fixed $\lambda\ge0$, maximizing
$\widehat R(S)-\lambda\widehat T(S)$ over ancestor-closed $S$ is then a
maximum-weight closure problem \citep{picard1976}. Its network has a vertex
for each node and each pair, an arc of capacity $p_v$ from the source to each
node, arcs of capacity $\lambda a$ from each node and $\lambda\beta$ from each
pair to the sink, and infinite arcs from each node to its parent and to its
pairs; the source side of a minimum cut is an optimal subtree with its pairs.
Source capacities are constant and sink capacities are nondecreasing in
$\lambda$, which after reversing all arcs is the form handled by
\citet{gallo1989}: one parametric maximum flow yields nested minimum cuts for
all $\lambda$, at most as many as the network has vertices. At $\lambda$ equal
to the largest ratio, the maximum of $\widehat R-\lambda\widehat T$ is zero and
is attained only by subtrees of that ratio, so the best ratio among the nested
cuts is the largest ratio.
\end{proof}

\paragraph{Exact search.}
Algorithm~\ref{alg:selection} gives the search of
Section~\ref{sec:selection}. Pool node IDs put parents before descendants, and
the depth-first traversal adds nodes in increasing ID order, each only when
its parent is present, so it reaches every subtree in $\mathcal A_K$ exactly
once.

\begin{algorithm}[htbp]
\caption{SpecFit selector}
\label{alg:selection}
\begin{algorithmic}[1]
\REQUIRE Pool $V$, root $s$, atlas $\widehat E$, frozen coefficients $(b,\beta)$
\STATE Read cumulative draft probabilities $p_v$ and encode expert bitmaps $B_v$
\STATE Initialize the incumbent to the root-only choice
\FORALL{ancestor-closed $S\subseteq V$, $|S|\leq\min(7,|V|)$, in depth-first order}
  \STATE Carry $B(S)=B_s\mathbin{\mathrm{OR}}\bigvee_{v\in S}B_v$ from the parent state
  \STATE $r\gets 1+\sum_{v\in S}p_v$; $f\gets\operatorname{popcount}(B(S))$
  \STATE $t\gets b_{|S|}+\beta f$
  \STATE Update incumbent by $(r/t,r,-f,-|S|)$ and the node-ID tie rule
\ENDFOR
\RETURN Incumbent subtree
\end{algorithmic}
\end{algorithm}

\subsection{Evaluation on Recorded Pools}

We evaluate these quantities on the candidate pools that SpecFit drafted in
every cycle of its runs with the expert cache (Appendix~\ref{app:calibration}),
with the coefficients fitted with the cache: 4,389 cycles on the 40 summarization articles
and 4,407 cycles on the 80 GSM8K questions. The legal family of a cycle is
that of its drafted pool, expectations are taken over the recorded cycles with
the trajectory held fixed, and $\lambda^\ast$ is computed by Dinkelbach's
iteration (Table~\ref{tab:theory}). The optimality gap of
Equation~\ref{eq:objective} is below 1\% of $\lambda^\ast$ and lies within the
bound of Equation~\ref{eq:margin}, although the two rules select different
trees in about a quarter of the cycles. With the recorded acceptances in place
of $\widehat R$, $S^\ast$ commits 0.5\% and 0.7\% more tokens per unit of
predicted time than $\widehat S$. Doubling $\beta$ lowers both the expert
demand per committed token and the number of selected nodes. Drafting stops
before level 5 in 85.8\% of the summarization cycles and 33.7\% of the GSM8K
cycles.

\begin{table}[htbp]
\centering
\small
\caption{Selection theory evaluated on the candidate pools drafted by
SpecFit, with the coefficients fitted with the cache. Percentages of $\lambda^\ast$ refer to
the predicted throughput.}
\label{tab:theory}
\begin{tabular}{lcc}
\toprule
Quantity & Summarization & GSM8K \\
\midrule
Cycles & 4,389 & 4,407 \\
$\lambda^\ast-\Theta$ (\% of $\lambda^\ast$) & 0.74 & 0.95 \\
Bound of Equation~\ref{eq:margin} (\% of $\lambda^\ast$) & 1.38 & 1.75 \\
Cycles with $S^\ast\neq\widehat S$ (\%) & 26.6 & 24.5 \\
Coefficient of variation of $\rho$ & 0.31 & 0.25 \\
Gain of $S^\ast$ with recorded acceptances (\%) & 0.53 & 0.69 \\
Mean $F(\widehat S)/\widehat R(\widehat S)$ at $\beta$ / $2\beta$ & 365.5 / 351.5 & 305.7 / 298.0 \\
Mean selected nodes at $\beta$ / $2\beta$ & 2.88 / 2.35 & 5.09 / 4.54 \\
\midrule
Mean drafted levels & 2.86 & 4.39 \\
Cycles stopped before level 5 (\%) & 85.8 & 33.7 \\
\bottomrule
\end{tabular}
\end{table}

\section{Experimental Details}
\label{app:calibration}

\paragraph{Model and data.}
The target model has 48 MoE layers with 128 experts per layer and top-8
routing. The summarization workload consists of the 40 even-numbered of the 80
CNN/DailyMail articles in Spec-Bench, with inputs of 156--1,496 tokens and a
limit of 2048 input plus output tokens. The GSM8K workload is the 80-question
subset of Spec-Bench, with a limit of 512 input plus output tokens.

\paragraph{Backend and baselines.}
The offload backend replaces the expert modules of the Hugging Face
Transformers Qwen3-MoE model. With the expert cache, missing experts are
loaded on demand, with at most two copies in flight. A configuration without
the expert cache clears the cache before every target forward, including
prefill, and the clearing is timed; an expert shared by tokens of one forward
is still loaded once, and the KV cache is kept. EAGLE-3 drafts the full
five-level pool with the same generator as SpecFit and verifies all nodes
when the pool has fewer than seven nodes near the generation limit.

\paragraph{Timing and aggregation.}
Each method's request stream begins with an empty expert cache, whose
contents carry over between questions, while KV and draft state are new for
every question. Measured request time includes prefill, drafting, atlas
conversion, selection, target verification, commit, and required
synchronization; model loading, atlas acquisition, and
post-request scoring are excluded. For output count $N_i$,
request time $t_i^{\rm req}$, and prefill time $t_i^{\rm pre}$, the throughput
measures are
\begin{equation}
\mathrm{TPS}_{\rm dec}=\frac{\sum_i(N_i-1)}{\sum_i(t_i^{\rm req}-t_i^{\rm pre})},
\qquad
\mathrm{TPS}_{\rm req}=\frac{\sum_iN_i}{\sum_i t_i^{\rm req}}.
\label{eq:tps}
\end{equation}
The first output token is attributed to prefill. Draft nodes, committed
tokens, and ms/cycle are means per target verification cycle; an AR cycle is
one decode forward and commits one token. H2D/token counts expert-weight
transfers per output token, including prefill.

\paragraph{Calibration.}
The probes use the states of GSM8K questions 401--403 at cycle boundaries 0, 4,
and 8. In
each of these nine states, two probe trees are measured for each size
$k=0,\ldots,7$, a probability-prefix tree and a minimum-expert-union tree,
twice each and without clearing the expert cache. Averaging the repeats gives
$9\times8\times2=144$ points $(k_i,F_i,y_i)$, where $y_i$ is cycle time without
selector time. The fit minimizes $\sum_i(y_i-b_{k_i}-\beta F_i)^2$ subject to
$\beta,b_0,\ldots,b_7\geq0$. The same probe trees are also timed with the
expert cache cleared before every target forward, and SpecFit without the
cache uses the fit to these 144 points. The RTX 4090 of
Section~\ref{sec:device} repeats this protocol with 32 expert slots per layer.
In milliseconds, $\beta=0.210715513$ with the cache and $0.408001022$ without
it on the A100, and $0.206037121$ and $0.781593955$ on the RTX 4090; the
intercepts $b_k$ are listed below.
\begin{center}
\small
\setlength{\tabcolsep}{3pt}
\begin{tabular}{llrrrrrrrr}
\toprule
& $k$ & 0 & 1 & 2 & 3 & 4 & 5 & 6 & 7 \\
\midrule
A100 & With cache & 142.214 & 147.532 & 150.522 & 152.269 & 152.743 & 163.204 & 155.175 & 156.068 \\
& Without cache & 122.702 & 117.975 & 118.338 & 116.389 & 115.309 & 114.770 & 113.887 & 113.671 \\
\midrule
RTX 4090 & With cache & 245.640 & 293.125 & 319.436 & 343.284 & 367.087 & 386.660 & 401.359 & 416.368 \\
& Without cache & 151.724 & 158.827 & 164.033 & 164.883 & 169.941 & 171.869 & 173.418 & 175.833 \\
\bottomrule
\end{tabular}
\end{center}
The intercepts are fitted independently and need not increase with $k$.
Reproduction uses the full-precision frozen coefficients.

\paragraph{Measurement runs.}
On GSM8K, AR and EAGLE-3 with the expert cache come from the main run on the
80 questions, and SpecFit with the cache from a second run; a third run
measures AR and EAGLE-3 without the expert cache, and a fourth SpecFit
without it. On summarization, AR without the cache and EAGLE-3 and SpecFit
with the cache come from one run on the 40 articles, AR with the cache and
EAGLE-3 without it from a second run, and SpecFit without the cache from a
third run; the slot sweep runs EAGLE-3 and SpecFit once at 24 and once at 12
slots. On the RTX 4090, SpecFit with the cache comes from one run on the 40
questions, AR without the cache and EAGLE-3 with and without it from a second
run, SpecFit without the cache from a third run, and AR with the cache from a
fourth. Every configuration has 80 requests on GSM8K on the A100 and 40
otherwise.

\end{document}